\documentclass[letterpaper, 10 pt, conference]{ieeeconf}  

\IEEEoverridecommandlockouts 

\usepackage{amsmath,amsfonts}
\usepackage{graphicx}
\usepackage{textcomp}
\usepackage{comment}

\usepackage{amsthm}
\usepackage[utf8]{inputenc}
\usepackage[T1]{fontenc}
\usepackage{enumerate}
\usepackage{soul}
\usepackage[dvipsnames,table]{xcolor} 
\usepackage{hyperref}
\usepackage{siunitx}
\usepackage{multirow}
\usepackage{subfig}
\usepackage{booktabs}
\usepackage[ruled,vlined]{algorithm2e}
\usepackage{algpseudocode}
\DontPrintSemicolon

\usepackage{amssymb}    
\usepackage{mathtools}  
\usepackage{bm}         
\usepackage{accents}
\usepackage{caption}
\usepackage{colortbl}  
\usepackage{array}     
\usepackage{adjustbox} 
\usepackage{wasysym}  
\usepackage{pifont}
\usepackage{stfloats}
\usepackage{url}
\usepackage{verbatim}
\usepackage{makecell}
\usepackage{tikz}
\usepackage{color}  
\usepackage{lipsum} 
\usepackage{float}

\usepackage[style=ieee,backend=biber,sortcites=true]{biblatex}

\usepackage{fontawesome5}

\renewcommand*{\multicitedelim}{\addcomma\space}
\DeclareCiteCommand{\cite}[\mkbibbrackets]
  {\usebibmacro{prenote}}
  {\usebibmacro{citeindex}\printfield{labelnumber}}
  {\multicitedelim}
  {\usebibmacro{postnote}}
  
\AtBeginBibliography{\scriptsize}  
\definecolor{blue}{HTML}{4E95D9}
\definecolor{gold}{HTML}{E5AE1C}
\definecolor{safe}{HTML}{4169E1}
\definecolor{unsafe}{HTML}{C41230}

\newtheorem{theorem}{Theorem}

\newtheorem{lemma}{Lemma}
\newtheorem{definition}{Definition}
\newtheorem{assumption}{Assumption}
\newtheorem{remark}{Remark}

\DeclareMathOperator*{\argmin}{argmin}

\usetikzlibrary{shapes, arrows, arrows.meta, positioning, shapes.geometric, decorations.pathreplacing, calc, patterns}

\definecolor{lbacbf}{HTML}{2E79BF}
\newcommand{\lbacbf}{\textbf{\texttt{\textcolor{lbacbf}{LBA-CBF}}}}

\definecolor{cOracle}{HTML}{6E6E6E}
\definecolor{cRaCBF} {HTML}{6A3D9A}
\definecolor{cLone}  {HTML}{CC79A7}
\definecolor{cDOB}   {HTML}{009E73}
\definecolor{cRCBF}  {HTML}{8C564B}
\definecolor{cNom}   {HTML}{FFA500}
\definecolor{nomGray}{HTML}{7F7F7F}

\AtEveryBibitem{%
  \clearfield{doi}\clearfield{url}\clearfield{urlyear}\clearfield{urlmonth}\clearfield{urlday}%
  \clearfield{isbn}\clearfield{issn}\clearfield{eprint}\clearfield{month}%
  \clearlist{language}%
  \ifentrytype{inproceedings}{%
    \clearlist{publisher}\clearlist{organization}\clearlist{location}%
    \clearname{editor}\clearfield{series}\clearfield{volume}}{}%
}

\hypersetup{hidelinks}

\begin{document}

\title{\LARGE \bf \lbacbf: Rapidly Adaptive Safety Filters \\ via Parallel Dynamics Inference}

\author{Maitham F. AL-Sunni%
\quad\quad\quad Timeea-Andreea Radu%
\quad\quad\quad Hassan Almubarak%
\quad\quad\quad Henry Z. Liao%
\\%
\quad\quad\quad Michael Görner%
\quad\quad\quad Francesco Maurelli%
\quad\quad\quad John M. Dolan%
\thanks{M. F. AL-Sunni is with the Department of Electrical \& Computer Engineering, Carnegie Mellon University, Pittsburgh, PA, USA. Email: {\tt  malsunni@andrew.cmu.edu}}%
\thanks{T.~Radu, M. Görner, and F. Maurelli are with the School of Computer Science \& Engineering, Constructor University, Bremen, Germany. Email: {\tt  \{tradu,migoerner,fmaurelli\}@constructor.university}}%
\thanks{H. Almubarak is with the Control \& Instrumentation Engineering Department, King Fahd University of Petroleum and Minerals, Dhahran, Saudi Arabia. Email: {\tt halmubarak@kfupm.edu.sa}}%
\thanks{H. Z. Liao and J. M. Dolan are with the Robotics Institute, Carnegie Mellon University, Pittsburgh, PA, USA. Email: {\tt \{hzl,jdolan\}@andrew.cmu.edu}}%
}

\maketitle

\begin{abstract}
Control barrier functions (CBFs) certify commands through an assumed dynamics model, so an abrupt, unmeasured regime change can undermine the certificate exactly when safety matters most. We present \textcolor{lbacbf}{Look-Back Adaptive Control Barrier Functions (\lbacbf{})}, which rank a finite bank of candidate dynamics by recent prediction error over a short look-back window and enforce the high-order CBF condition against every model within a tolerance of the best, spanning best-fit adaptation to full-bank robust filtering. The dynamics may depend nonlinearly on the unknown parameters, and no switching model or continuously parameterized estimator is required. We prove that any feasible filtered input satisfies the true CBF condition whenever a safety-representative candidate is retained. In quadrotor simulation with abrupt wind reversals and an unknown payload, \lbacbf{} is safe and reaches the goal from all random initial conditions, matching an oracle, while adaptive and robust baselines achieve $0$--$88\%$ success. Banks of up to $250,000$ models run inside the control loop, and Crazyflie~2.1 and F1TENTH experiments demonstrate adaptation to wind, payload release, and varying tire--road friction. Code, videos,
and project details are available at:
\vspace{-3pt}
\begin{center}    
\href{https://lla-control.github.io}{\textcolor{lbacbf}{https://lla-control.github.io}}
\end{center}
\end{abstract}

\section{Introduction}

A control safety filter is only as trustworthy as the dynamics model used inside its barrier condition. Control barrier functions (CBF) \cite{ames2017cbf} keep the state inside a safe set with minimal deviation from a nominal controller, but certify each input through the predicted evolution of the barrier under an assumed dynamics model. An abrupt, unmeasured change of \emph{dynamics regime}, such as an unknown payload, a sudden wind gust, or a change in tire--road friction, therefore undermines the filter exactly when safety matters most.

Two main remedies exist. Robust CBFs certify the barrier against a fixed worst-case uncertainty set~\cite{jankovic2018rcbf}, which can be conservative in a way that online data cannot reduce. Adaptive and learning-based CBFs reduce this conservatism by estimating unknown parameters or model errors online, but abrupt regime changes must be absorbed through an estimation transient. Robust adaptive CBFs combine adaptation with set-based uncertainty descriptions~\cite{lopez2021racbf}, but under the standard constant-parameter formulation the admissible parameter set contracts as measurements accumulate and does not re-expand when the dynamics change. Safety certification thus remains tied either to a fixed uncertainty description or to the transient of a continuously updated estimate.

We propose \textcolor{lbacbf}{\bf Look-Back Adaptive Control Barrier Functions (\lbacbf)}, whose architecture is illustrated in \autoref{fig:framework}. Rather than committing the safety filter to a single continuously evolving estimate, \lbacbf\ maintains a finite bank of candidate dynamics models and, at every control step, evaluates them in parallel by their prediction error over a look-back window of length $W$. The best-fitting model can be used directly for rapid adaptation, while a prediction-plausible set retains models whose recent errors lie within a tolerance of the minimum and enforces the CBF condition for all retained models. After $W$ new transitions, pre-change observations no longer contribute to the residual, so the filter adapts to changes without a switching model. 

\begin{figure}
    \centering
    \includegraphics[width=\linewidth]{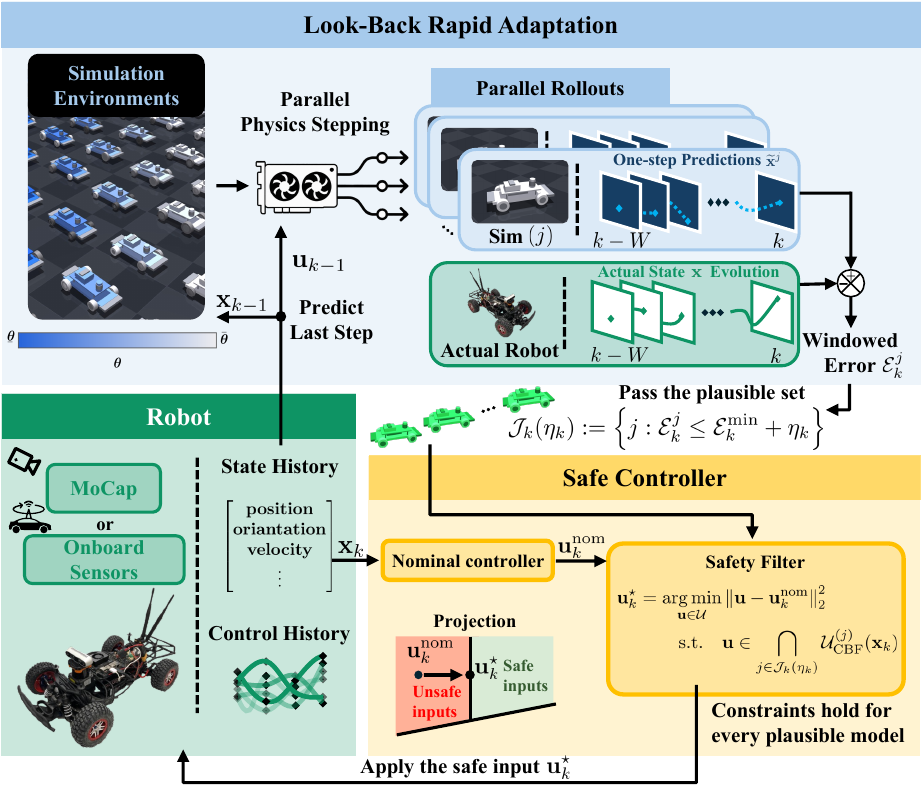}
    \caption{Overview of the \lbacbf{} framework. The look-back
    stage ranks candidate models by recent prediction error
    and retains a prediction-plausible set, while the look-ahead computes a control input satisfying the CBF constraints
    for every retained model.}
    \label{fig:framework}
    \vspace{-15pt}
\end{figure}

\subsection{Related Work}

\subsubsection{Robust and adaptive safety}
CBFs have been extended to high-relative-degree and discrete-time safety constraints~\cite{xiao2019hocbf,xiong_discrete-time_2023}. Under model uncertainty, robust CBFs certify safety against a prescribed uncertainty set~\cite{jankovic2018rcbf}, while adaptive CBFs reduce conservatism through online parameter adaptation~\cite{taylor2020adaptive}. Robust adaptive CBFs (RaCBFs) combine adaptation with set-membership identification~\cite{lopez2021racbf}, with extensions to high-relative-degree constraints~\cite{cohen2022hocabf}. Alternatively, safety can be embedded into an adaptive controller by augmenting the parametrically uncertain dynamics with barrier states~\cite{alsunni2025seac}. More recent approaches accommodate time-varying uncertainty: robust CBFs in~\cite{pati2025parametervarying} admit time-varying parameters that enter the dynamics nonlinearly and use set-membership estimation to reduce conservatism, while~\cite{diget2026timevarying} combines finite-time estimation of time-varying parameters with RaCBFs under a structured parameterization and bounded parameter variation. Disturbance-observer~\cite{wang2023dobcbf} and $\mathcal{L}_1$~\cite{nguyen2022l1cbf} approaches instead estimate and compensate model mismatch through observer or adaptive dynamics. In contrast, \lbacbf\ does not require a continuously parameterized estimator or an affine-in-parameter uncertainty model: it ranks a finite bank of candidate dynamics directly from recent prediction error and certifies the barrier against the dynamics hypotheses retained by the moving look-back window.

\subsubsection{Switched dynamics and regime change}
CBF methods for hybrid and switched systems construct mode-dependent safety conditions under a specified hybrid structure~\cite{yang2024local}, while learning-based approaches have estimated mode-specific barrier residuals for switching systems~\cite{aali2024piecewise}. In contrast, \lbacbf\ treats the active dynamics regime as latent and infers it from recent state--action transitions without requiring known guards or a transition model.

\subsubsection{Multiple-model methods}
Adaptation through a finite bank of candidate dynamics is a classical multiple-model control idea~\cite{narendra1997multiple}. More recently, multiple-model reasoning has appeared in safety-critical control:~\cite{laouar2024jmls} combines CBF constraints with MPC for jump Markov linear systems under uncertain mode estimates, while~\cite{li2024supervisory} uses a supervisory multi-observer bank to estimate unknown parameters within an adaptive safety controller. Most closely related to our inference mechanism, the look-back bank of~\cite{alsunni2025llampc,liao2026llampc_embedded} ranks candidate vehicle models over a short window and selects the best-fitting model for adaptive racing model predictive control (MPC). \lbacbf\ instead uses the windowed prediction residual for safety certification, retaining the best-fitting dynamics or a prediction-plausible subset and enforcing the high-order CBF (HOCBF) against them.

\subsection{Contributions}
The main contributions of this work are:
\begin{itemize}

\item \textbf{Look-back adaptive safety filtering.} \lbacbf\ ranks a finite bank of candidate dynamics by recent prediction error and enforces the CBF condition against the best-fitting models. The parameters may enter the dynamics nonlinearly, including through the input gain, and no switching model or continuously parameterized estimator is required.

\item \textbf{Safety guarantee and adaptive-to-robust filtering.} Every feasible \lbacbf\ input satisfies the CBF condition of the true dynamics whenever the retained set contains a safety-representative model. A finite tolerance that guarantees this exists whenever the bank contains such a model, and the same tolerance spans best-fit adaptation and full-bank robust filtering.

\item \textbf{Real-time implementation and robotic validation.} A parallel JAX implementation runs banks of up to $250,000$ models inside the control loop. In quadrotor simulation, \lbacbf\ is safe and reaches the goal from all random initial conditions, outperforming adaptive and robust baselines. Experiments on a Crazyflie~2.1 and an F1TENTH race car demonstrate adaptation to wind, a swinging payload with an abrupt release, and varying tire--road friction.

\end{itemize}

\section{Problem Formulation}
\label{sec:problem}

\subsection{Robotic systems with changing dynamics regimes}
\label{sec:changing_dynamics}

Consider the discrete-time robotic system
\begin{equation} \label{eq: system parametric dynamics}
\mathbf{x}_{k+1} =  F(\mathbf{x}_{k},\mathbf{u}_{k};\boldsymbol{\theta}_{k}) :=  f(\mathbf{x}_{k};\boldsymbol{\theta}_{k})  + g(\mathbf{x}_{k};\boldsymbol{\theta}_{k})\mathbf{u}_{k}
\end{equation}
where $k\in\mathbb{Z}_{\geq 0}$, $\mathbf{x}_{k}\in\mathcal{X}\subseteq\mathbb{R}^{n}$, $\mathbf{u}_{k}\in\mathcal{U}\subseteq\mathbb{R}^{m}$, and $\boldsymbol{\theta}_{k}\in\Theta\subseteq\mathbb{R}^{p}$ denote the time index, state, control input, and unobserved dynamics parameter, respectively. We assume that $\Theta$ and $\mathcal{U}$ are compact.

$\boldsymbol{\theta}_{k}$ parametrizes $F$ to represent the active dynamics regime of the robot. Such regimes arise from changes in payload, contact, terrain friction, or aerodynamic loading. We do not assume that the active regime is directly observed or that its evolution follows a known switching law. Moreover, $F(\mathbf{x},\mathbf{u};\boldsymbol{\theta})$ need not be affine in the unknown parameter $\boldsymbol{\theta}$.

Safety of the robot is represented by the path constraint $h:\mathcal{X}\rightarrow\mathbb{R}$, which defines the safe set
\begin{equation} \label{eq: safe set def}
\mathcal{S} :=  \left\{\mathbf{x}\in\mathcal{X}: h(\mathbf{x})\geq 0 \right\}
\end{equation}

A controller is \emph{safe} if every closed-loop trajectory from an admissible initial condition satisfies $h(\mathbf{x}_{k})\geq 0$ for all $k\in\mathbb{Z}_{\geq 0}$, i.e., it renders $\mathcal{S}$ controlled invariant.

\subsection{Discrete-time high-order control barrier functions}
\label{sec:hocbf_preliminaries}

We use discrete-time HOCBFs to
construct the safety filter, making the dependence on the dynamics
parameter explicit, since the filter will evaluate this
condition under different candidate models.

We say that $h$ has relative degree $r\geq 1$ for
\eqref{eq: system parametric dynamics} if the input first affects the
safety function after $r$ steps, i.e., $h(\mathbf{x}_{k+\ell})$ is
independent of the applied inputs for $\ell < r$ and depends on
$\mathbf{u}_{k}$ for $\ell = r$.

\begin{definition}[Discrete-time HOCBF \cite{agrawal2017dcbf,xiao2019hocbf,zeng2021mpccbf}]
\label{def:dcbf}
Let $\alpha_{1},\ldots,\alpha_{r}\in(0,1]$. For a parameter
$\boldsymbol{\vartheta}\in\Theta$, define
$\psi_{0}(\mathbf{x};\boldsymbol{\vartheta}) := h(\mathbf{x})$ and, for
$q=1,\ldots,r$,
\begin{equation}
\psi_{q}(\mathbf{x},\mathbf{u};\boldsymbol{\vartheta})
:=
\psi_{q-1}\bigl(F(\mathbf{x},\mathbf{u};\boldsymbol{\vartheta});\boldsymbol{\vartheta}\bigr)
- (1-\alpha_{q})\,\psi_{q-1}(\mathbf{x};\boldsymbol{\vartheta}).
\label{eq:hocbf_recursion}
\end{equation}
By the relative-degree property the right-hand side is independent of
$\mathbf{u}$ for $q\leq r-1$, so $\psi_{1},\ldots,\psi_{r-1}$ are functions of
the state alone and we write $\psi_{q}(\mathbf{x};\boldsymbol{\vartheta})$;
$\psi_{r}$ is the first term in which the input appears. The order-$r$ HOCBF
condition at time $k$ is
\begin{equation}
\psi_{r}(\mathbf{x}_{k},\mathbf{u};\boldsymbol{\vartheta}) \;\geq\; 0.
\label{eq:hocbf_condition}
\end{equation}
If \eqref{eq:hocbf_condition} is enforced at every step and
$\psi_{q}(\mathbf{x}_{0};\boldsymbol{\vartheta})\geq 0$ for
$q=0,\ldots,r-1$, then the safe set is forward
invariant~\cite{xiong_discrete-time_2023}.
\end{definition}

For $r=1$, \eqref{eq:hocbf_condition} reduces to the relative-degree-one
discrete-time CBF condition
$h\left(F(\mathbf{x}_{k},\mathbf{u};\boldsymbol{\vartheta})\right)
- (1-\alpha_{1})h(\mathbf{x}_{k}) \geq 0$.

\noindent \textbf{Problem.} Given the uncertain system \eqref{eq: system parametric dynamics}, the safe set $\mathcal{S}$, a nominal controller $\mathbf{u}_{k}^{\mathrm{nom}} = \pi_{\mathrm{nom}}(\mathbf{x}_{k})$ and recent observed state--action transitions, our objective is to construct a real-time safety filter that minimally modifies $\mathbf{u}_{k}^{\mathrm{nom}}$ while accounting for the unobserved and potentially changing dynamics regime $\boldsymbol{\theta}_{k}$.

\section{\lbacbf: Rapidly Adaptive Safety Filters}
\label{sec:approach}

The proposed \lbacbf\ framework consists of three stages. First, a finite bank of candidate dynamics models is constructed over the operating regimes. Second, a look-back module evaluates how well each candidate explains the most recent transitions and retains a set of prediction-plausible models. Third, the safety filter enforces the CBF condition against every retained model and returns the admissible input closest to the nominal command.

\subsection{Finite dynamics model bank}
\label{sec:bank}

Let $\Theta_{\mathcal{B}} :=
\{\boldsymbol{\theta}^{1},\dots,\boldsymbol{\theta}^{M}\} \subseteq \Theta$
be a finite bank of candidate parameter values. Each
$\boldsymbol{\theta}^{j}$ induces a candidate transition map, yielding the
finite family of dynamics models
\begin{equation}
\mathcal{B}
:= 
 \left\{
    F(\mathbf{x}, \mathbf{u};\, \boldsymbol{\theta}^{j})
  \right\}_{j=1}^{M} = \left\{
    F^{j}(\mathbf{x}, \mathbf{u})
  \right\}_{j=1}^{M}.
\label{eq:model_bank}
\end{equation}
The bank may be constructed from physically meaningful parameter ranges,
system-identification results, expert-designed operating modes,
simulation-based domain randomization, or other finite collections of
candidate dynamics. In our experiments, we use structured grids over
prescribed parameter ranges because they provide simple and reproducible
coverage of the operating regimes. However, the
\lbacbf\ construction itself requires only a finite candidate bank and is
not tied to a particular sampling strategy.

\subsection{Look-back model plausibility}
\label{sec:lookback}

At every control step, the bank produces one-step predictions in parallel
from the \emph{measured} previous transition: the pair
$(\mathbf{x}_{k-1},\mathbf{u}_{k-1})$ is broadcast to all $M$ candidate
models and stepped in one batched evaluation,
\begin{equation}
  \mathcal{B}_{k}
  =
  \left\{
    F^{j}(\mathbf{x}_{k-1}, \mathbf{u}_{k-1})
  \right\}_{j=1}^{M}
  =
  \left\{ \widetilde{\mathbf{x}}^{j}_{k} \right\}_{j=1}^{M},
  \label{eq:bank_predictions}
\end{equation}
where $\mathcal{B}_{k}$ is the bank evaluated at
step $k$ and $\widetilde{\mathbf{x}}_{k}^{j}$ is the prediction of candidate $j$. Once
$\mathbf{x}_{k}$ is measured, each candidate is scored by its quadratic
one-step prediction error
\begin{equation}
e_{k}^{j}
:=
\bigl\| \mathbf{x}_{k} - \widetilde{\mathbf{x}}_{k}^{j} \bigr\|_{\Lambda_{k}}^{2}
= (\mathbf{x}_{k}-\widetilde{\mathbf{x}}_{k}^{j})^{\!\top}\Lambda_{k}(\mathbf{x}_{k}-\widetilde{\mathbf{x}}_{k}^{j}).
\label{eq:one_step_prediction_error}
\end{equation}
Here $\Lambda_{k}$ is diagonal with positive entries and weights each state by its recent
innovation scale, so that dimensions with different units contribute comparably;
$\Lambda_{k}=\mathbf{I}$ recovers the unweighted residual.
The look-back module aggregates these scores over the most recent $W$
transitions through the windowed prediction loss
\begin{equation}
\mathcal{E}_{k}^{j} := \mathcal{E}_{k} (\boldsymbol{\theta}^{j}) = \sum_{i=k-W+1}^{k}e_{i}^{j},
\label{eq:windowed_residual}
\end{equation}
which is maintained recursively as $\mathcal{E}_{k}^{j} = \mathcal{E}_{k-1}^{j}
+ e_{k}^{j} - e_{k-W}^{j}$. This recursive structure allows each candidate's score to be updated in constant time
per step and the look-back module in $O(M)$ per step, independent of the
window length (for $k \leq W$, the sum accumulates without subtraction).

The window length $W\in\mathbb{Z}_{\geq 1}$ controls the tradeoff between
responsiveness to changes in the active regime and noise rejection. A
longer window accumulates prediction errors over more transitions and
therefore produces a smoother residual landscape but reacts more slowly to
a change in the active dynamics. A shorter window, on the other hand,
responds more rapidly to regime changes, at the cost of greater sensitivity
to measurement and process noise.

A natural regime estimate is the empirical-risk minimizer
\begin{align}
\widehat{j}_{k} \in \argmin_{j\in\{1,\ldots,M\}} \mathcal{E}_{k}^{j}, \quad {\text{ with }} \quad
\widehat{\boldsymbol{\theta}}_{k} :=
\boldsymbol{\theta}^{\widehat{j}_{k}}.
\label{eq:theta_hat}
\end{align}

For safety filtering, however, collapsing the residual landscape to a
single minimizer can be overly optimistic. Over a finite look-back window,
multiple candidate models may explain the observed transitions almost
equally well and hence the recent data may not support a unique regime
selection. More importantly, the residual criterion measures predictive
consistency with past transitions, whereas safety depends on the CBF
evolution induced by the candidate dynamics. Consequently, the
minimum-residual model need not be the model most relevant to the current
safety constraint.

Therefore, we retain the model ambiguity that remains unresolved by the
recent data rather than committing to a single candidate. Let
\begin{equation}
\mathcal{E}_{k}^{\min} := \min_{j\in\{1,\ldots,M\}} \mathcal{E}_{k}^{j}.
\label{eq:E_min}
\end{equation}
Then, for a tolerance $\eta_{k}\geq 0$, we define the
\emph{prediction-plausible set}
\begin{equation}
\mathcal{J}_{k}(\eta_{k}) := \left\{ j\in\{1,\ldots,M\} : \mathcal{E}_{k}^{j} \leq \mathcal{E}_{k}^{\min} + \eta_{k} \right\}.
\label{eq:plausible_set}
\end{equation}

\subsection{Multi-model \lbacbf\ safety filter}
\label{sec:lbacbf_filter}

For each candidate model $j$, define its model-dependent CBF value at
time $k$ by
\begin{equation}
H_{k}^{j}(\mathbf{u}) := \psi_{r} \left( \mathbf{x}_{k}, \mathbf{u}; \boldsymbol{\theta}^{j} \right),
\label{eq:model_barrier_value}
\end{equation}
where $\psi_{r}$ is generated according to \autoref{def:dcbf}. Thus,
$H_{k}^{j}(\mathbf{u})\geq 0$ is precisely the CBF condition that would
be imposed if $\boldsymbol{\theta}^{j}$ were the active dynamics regime.

The \lbacbf\ safety filter enforces the model-dependent CBF condition
for every candidate in the plausible set:
\begin{align}
\mathbf{u}_{k}^{\star} \in \argmin_{\mathbf{u}\in\mathcal{U}} \quad &
\frac{1}{2} \left\| \mathbf{u} - \mathbf{u}_{k}^{\mathrm{nom}} \right\|_{2}^{2},
\label{eq:lbacbf_filter}
\\
\text{s.t.} \quad & H_{k}^{j}(\mathbf{u}) \geq 0, \qquad \forall j\in\mathcal{J}_{k}(\eta_{k}). \nonumber
\end{align}
When
the model-dependent CBF constraints are affine in $\mathbf{u}$,
\eqref{eq:lbacbf_filter} reduces to a quadratic program (QP); more generally,
the same construction defines a constrained safety-filter optimization.
We define the corresponding feasible control set
\begin{equation}
\hspace{-3mm}
\mathcal{U}_{k}^{\mathrm{LBA}}(\eta_{k}) := \left\{\mathbf{u}\in\mathcal{U} : H_{k}^{j}(\mathbf{u})\geq 0, \quad
\forall j\in\mathcal{J}_{k}(\eta_{k}) \right\}.
\label{eq:lba_feasible_set}
\end{equation}

\begin{remark}[Adaptive-to-robust continuum]
\label{rem:eta_continuum}
If $0\leq\eta_{k}^{1}\leq\eta_{k}^{2}$, then
$\mathcal{J}_{k}(\eta_{k}^{1}) \subseteq \mathcal{J}_{k}(\eta_{k}^{2})$
and consequently
$\mathcal{U}_{k}^{\mathrm{LBA}}(\eta_{k}^{2}) \subseteq
\mathcal{U}_{k}^{\mathrm{LBA}}(\eta_{k}^{1})$: increasing $\eta_{k}$ only
adds model-dependent CBF constraints and weakly shrinks the admissible
control set. Letting
$\mathcal{E}_{k}^{\max} := \max_{j\in\{1,\ldots,M\}} \mathcal{E}_{k}^{j}$,
the extreme $\eta_{k}=0$ recovers a best-fit, minimum-residual safety
filter, while any
$\eta_{k} \geq \mathcal{E}_{k}^{\max} - \mathcal{E}_{k}^{\min}$ yields
$\mathcal{J}_{k}(\eta_{k}) = \{1,\ldots,M\}$, i.e., the CBF condition is
enforced over the entire model bank. The tolerance $\eta_{k}$ therefore
interpolates between single-model adaptation and full-bank robust
filtering.
\end{remark}

Both stages of a control step parallelize across models: the $M$ prediction
losses are independent across candidates, as are the model-dependent CBF
evaluations over $\mathcal{J}_{k}(\eta_{k})$. We exploit this with vectorized
JAX~\cite{JAX}, which runs on CPU or GPU and keeps large banks inside the
control loop.

\subsection{Safety guarantee}
\label{sec:safety_guarantee}

The prediction-plausible set determines which candidates the filter accounts for, but prediction error alone does not connect a candidate to the true safety constraint, so we introduce a bank-adequacy condition in the barrier domain.

Define the CBF value of the true current dynamics as
\begin{equation}
H_{k}^{\mathrm{true}}(\mathbf{u}) := \psi_{r}\left(\mathbf{x}_{k},\mathbf{u};\boldsymbol{\theta}_{k}\right).
\label{eq:true_barrier_value}
\end{equation}

\begin{assumption}[Safety-representative bank]
\label{ass:safety_representative}
For every time $k$, the candidate bank contains at least one model $F^{j_{k}^{s}}$ whose CBF prediction is no more optimistic than that of the true current dynamics. Specifically,
\begin{equation}
H_{k}^{j_{k}^{s}}(\mathbf{u}) \leq H_{k}^{\mathrm{true}}(\mathbf{u}), \qquad \forall \mathbf{u}\in\mathcal{U}
\label{eq:safety_representative}
\end{equation}
\end{assumption}

\autoref{ass:safety_representative} does not require $\boldsymbol{\theta}_{k}$ to lie in the bank, only a candidate whose barrier prediction is conservative.

\begin{lemma}[Safety transfer]
\label{lemma:safety_transfer}
Suppose \autoref{ass:safety_representative} holds at time $k$, the safety-representative candidate satisfies $j_{k}^{s}\in\mathcal{J}_{k}(\eta_{k})$, and the safety-filter problem \eqref{eq:lbacbf_filter} is feasible. Then its solution $\mathbf{u}_{k}^{\star}$ satisfies the CBF condition for the true current dynamics: $H_{k}^{\mathrm{true}}(\mathbf{u}_{k}^{\star}) \geq 0$
\end{lemma}

\begin{proof}
Since $j_{k}^{s}\in\mathcal{J}_{k}(\eta_{k})$, the constraints of \eqref{eq:lbacbf_filter} give $H_{k}^{j_{k}^{s}}(\mathbf{u}_{k}^{\star}) \geq 0$, so by \autoref{ass:safety_representative}, $H_{k}^{\mathrm{true}}(\mathbf{u}_{k}^{\star}) \geq H_{k}^{j_{k}^{s}}(\mathbf{u}_{k}^{\star}) \geq 0$.
\end{proof}

\autoref{lemma:safety_transfer} reduces satisfaction of the true CBF condition to inclusion of a safety-representative candidate in the prediction-plausible set. We next establish that a sufficiently large tolerance always exists under \autoref{ass:safety_representative}.

\begin{lemma}[Existence of a sufficient plausible-set tolerance]
\label{lemma:eta_existence}
Let $F^{j_{k}^{s}}$ be a safety-representative candidate satisfying \autoref{ass:safety_representative}, and define
$\eta_{k}^{s} := \mathcal{E}_{k}^{j_{k}^{s}} - \mathcal{E}_{k}^{\min}$
Then,
\begin{equation}
0 \leq \eta_{k}^{s} \leq \mathcal{E}_{k}^{\max} - \mathcal{E}_{k}^{\min}
\label{eq:eta_safe_bounds}
\end{equation}
and every satisfying $\eta_{k}\geq\eta_{k}^{s}$ guarantees $j_{k}^{s}\in\mathcal{J}_{k}(\eta_{k})$.
\end{lemma}

\begin{proof}
Since $\mathcal{E}_{k}^{\min}$ is the minimum loss over the bank,$\eta_{k}^{s} = \mathcal{E}_{k}^{j_{k}^{s}} - \mathcal{E}_{k}^{\min} \geq 0$. Furthermore, $\mathcal{E}_{k}^{j_{k}^{s}}\leq\mathcal{E}_{k}^{\max}$, which gives \eqref{eq:eta_safe_bounds}. Thus, for any $\eta_{k}\geq\eta_{k}^{s}$, $\mathcal{E}_{k}^{j_{k}^{s}} = \mathcal{E}_{k}^{\min} + \eta_{k}^{s} \leq \mathcal{E}_{k}^{\min} + \eta_{k}$
Hence, by \eqref{eq:plausible_set}, $j_{k}^{s}\in\mathcal{J}_{k}(\eta_{k})$.
\end{proof}

\begin{remark}[Selection of $\eta_{k}$]
\label{rem:eta_selection}
\autoref{lemma:eta_existence} establishes the existence of a tolerance sufficient to retain a safety-representative model, but it does not prescribe an online value for that tolerance, i.e. $j_{k}^{s}$ and hence $\eta_{k}^{s}$ are not known to the controller. In this work, $\eta_{k}$ is therefore treated as a design parameter and selected empirically. Smaller values retain fewer dynamics hypotheses and generally yield a less restrictive safety filter, whereas larger values retain more unresolved model ambiguity and impose a larger collection of safety constraints.

It is worth noting that \eqref{eq:eta_safe_bounds} provides a conservative certified endpoint. In particular, any
\begin{equation}
\eta_{k}\geq\mathcal{E}_{k}^{\max}-\mathcal{E}_{k}^{\min}
\label{eq:eta_full_bank_safe}
\end{equation}
necessarily satisfies $\eta_{k}\geq\eta_{k}^{s}$ under \autoref{ass:safety_representative} and recovers the full-bank robust safety filter.
\end{remark}

The recursion of \autoref{def:dcbf} is defined for a fixed parameter, whereas
the active regime $\boldsymbol{\theta}_{k}$ varies with $k$. For $r=1$ this is
immaterial, since $\psi_{0}=h$ does not depend on the parameter. For $r\geq2$ the
intermediate functions $\psi_{1},\ldots,\psi_{r-1}$ are parameter-dependent, and
propagating their sign across a regime change requires the following condition.

\begin{assumption}[Regime-change compatibility]
\label{ass:regime_compat}
For every $k\in\mathbb{Z}_{\geq 0}$ and every $q=1,\ldots,r-1$,
$\psi_{q}(\mathbf{x}_{k+1};\boldsymbol{\theta}_{k})\geq 0$ implies
$\psi_{q}(\mathbf{x}_{k+1};\boldsymbol{\theta}_{k+1})\geq 0$.
\end{assumption}

\autoref{ass:regime_compat} constrains only the transition between consecutive
regimes: it requires that the change from $\boldsymbol{\theta}_{k}$ to
$\boldsymbol{\theta}_{k+1}$ does not by itself render an intermediate barrier
function negative at the state reached. It is vacuous for $r=1$.

\begin{theorem}[Conditional closed-loop safety of \lbacbf]
\label{theorem:closed_loop_safety}
Let $\psi_{q}(\mathbf{x}_{0};\boldsymbol{\theta}_{0})\geq 0$ for $q=0,\ldots,r-1$. Suppose that for every $k\in\mathbb{Z}_{\geq 0}$, \autoref{ass:safety_representative} and \autoref{ass:regime_compat} hold, the selected tolerance satisfies $\eta_{k}\geq\eta_{k}^{s}$, and the safety-filter problem \eqref{eq:lbacbf_filter} is feasible. Then the true order-$r$ CBF condition is satisfied at every time step:
\begin{equation}
H_{k}^{\mathrm{true}}(\mathbf{u}_{k}^{\star}) \geq 0, \qquad \forall k\in\mathbb{Z}_{\geq 0}
\label{eq:closed_loop_true_hocbf}
\end{equation}
Consequently, the safe set $\mathcal{S}$ is rendered forward invariant.
\end{theorem}

\begin{proof}
For any time $k$, $\eta_{k}\geq\eta_{k}^{s}$ and \autoref{lemma:eta_existence} imply $j_{k}^{s}\in\mathcal{J}_{k}(\eta_{k})$. Since the safety-filter problem \eqref{eq:lbacbf_filter} is feasible, \autoref{lemma:safety_transfer} then gives $H_{k}^{\mathrm{true}}(\mathbf{u}_{k}^{\star}) \geq 0$. For the forward-invariance conclusion, suppose $\psi_{q}(\mathbf{x}_{k};\boldsymbol{\theta}_{k})\geq 0$ for $q=0,\ldots,r-1$. Applying \eqref{eq:hocbf_recursion} to $H_{k}^{\mathrm{true}}(\mathbf{u}_{k}^{\star})\geq 0$ and descending the chain gives $\psi_{q}(\mathbf{x}_{k+1};\boldsymbol{\theta}_{k})\geq 0$ for $q=0,\ldots,r-1$, since $\alpha_{q}\in(0,1]$. \autoref{ass:regime_compat} carries this to the incoming regime, $\psi_{q}(\mathbf{x}_{k+1};\boldsymbol{\theta}_{k+1})\geq 0$ for $q=1,\ldots,r-1$, while $\psi_{0}=h$ is parameter-free. The induction therefore closes and $h(\mathbf{x}_{k})\geq 0$ for all $k$.
\end{proof}

\section{Experimental Validation}
\label{sec:experiments}
We evaluate \lbacbf\ in one simulation study and two hardware studies; \emph{a video of the full
hardware runs is on the project's website}. The simulation compares
\lbacbf\ with four adaptive and robust CBF baselines, together with a fixed-model
nominal CBF and an oracle given the true parameters, from identical initial
conditions on a task whose payload scales the input, and ablates the look-back
window $W$ and the bank size $M$. On
hardware, \lbacbf\ is compared with the same filter using a fixed model, so that
adaptation is the only difference. On a Crazyflie~2.1 quadrotor, it filters a
linear quadratic regulator (LQR) under fan-induced wind and under a swinging
payload that is then cut loose. On an F1TENTH race car~\cite{f1tenth2020}, we experiment with slippery plastic and rubber tires where the proposed work identifies the model of a CBF-constrained MPC. The car's tire parameters enter the Fiala model nonlinearly, so no linear-in-parameter adaptive/robust CBF~\cite{taylor2020adaptive,lopez2021racbf} applies.

\subsection{Quadrotor: Simulation Experiments}
\label{subsec:quad_sim}

\subsubsection{Task, dynamics, and filter}
The vehicle must descend from a start band to a goal at $(0,-2.7)\,\mathrm{m}$ through
five windy gates (\autoref{fig:quad_sim_traj}), each includes obstacles of radius
$0.32\,\mathrm{m}$ leaving a $0.84\,\mathrm{m}$ corridor. The corridors alternate between $x=-0.3$
and $x=+0.3\,\mathrm{m}$ on the rows $y\in\{2,1,0,-1,-2\}\,\mathrm{m}$, so the vehicle must move
laterally between gates while all obstacles are constrained. The $26$ initial
conditions are drawn uniformly at random over $x_0\in[-2,2]\,\mathrm{m}$ and
$y_0\in[2.50,2.85]\,\mathrm{m}$, with an initial descent of $0.4\,\mathrm{m/s}$ that reaches the
first gate while every estimator is still converging from
$\hat{\boldsymbol{\theta}}=\mathbf{0}$; runs last $40\,\mathrm{s}$ at $dt=0.02\,\mathrm{s}$. The nominal LQR regulates straight to the goal, so all obstacle avoidance comes from
the safety filter.

The vehicle is the planar quadrotor
\begin{align}
\label{eq:quadrotor_dynamics}
m\dot{v}_x &= (u_0+u_1)\sin\theta + w_x - c_{dx}v_x|v_x|, \nonumber \\
m\dot{v}_y &= (u_0+u_1)\cos\theta - mg + w_y - c_{dy}v_y|v_y|,  \\
J\ddot{\theta} &= \tfrac{\ell}{2}\,(u_1 - u_0), \nonumber
\end{align}
with $\dot{x}=v_x$, $\dot{y}=v_y$, rotor thrusts $(u_0,u_1)$, pitch $\theta$
(distinct from the parameter vector $\boldsymbol{\theta}$), mass $m$,
gravitational acceleration $g$, rotor separation $\ell$, and moment of
inertia $J$.

Four horizontal bands, bounded at $y=1.2$, $0$, and $-1.2\,\mathrm{m}$, carry a lateral
wind $w_x=\pm0.07\,\mathrm{N}$ whose sign alternates between bands, so each of the three
band crossings is an instantaneous, unannounced reversal of the
disturbance---the regime change of \autoref{sec:changing_dynamics}. Outside a
calm pocket of radius $0.3\,\mathrm{m}$ at the goal, where wind and drag vanish, the drag
coefficients are $c_{dx}=c_{dy}=0.10$, and everywhere an unknown payload raises
the mass by $\Delta m=30\%$, so that $m$ becomes $m(1+\Delta m)$ in
\eqref{eq:quadrotor_dynamics} and
$\boldsymbol{\theta}=(w_x,w_y,c_{dx},c_{dy},\Delta m)$. The payload therefore
scales the control-input term of \eqref{eq: system parametric dynamics},
nonlinearly in $\boldsymbol{\theta}$; the hardware bank lumps it into
$(w_x,w_y)$, whereas here it is modeled explicitly.

Each safety constraint is enforced by a second-order CBF on the collective thrust. Because
at $\theta\approx 0$ the two thrusts enter that condition only through their
sum, leaving no lateral authority, we add a fourth-order CBF on the
differential thrust, whose tilt bends the straight nominal reference around obstacles. Both are enforced at $50\,\mathrm{Hz}$. The QP is solved by sequential least-squares quadratic programming. Fixed-gain CBFs do not
guarantee feasibility under input limits; when the solver fails, which in every
case we examined occurred because no admissible input satisfied all conditions,
the implementation applies the nominal command.

The bank is a $15\times3\times6\times6\times7$ grid over $\boldsymbol{\theta}$
plus the zero model, $M=11\,341$, spanning $w_x,w_y\in[-0.33,0.33]\,\mathrm{N}$
($4.7\times$ the true wind), $c_{dx},c_{dy}\in[0,0.55]$, and
$\Delta m\in[0,0.72]$; no grid value coincides with the true wind, drag, or
$\Delta m$, although the wind axes contain the true $w_y=0$. We use $W=5$
($0.1\,\mathrm{s}$) and enforce the barrier against every model whose windowed residual
lies within $\eta$ of the lowest, as in \autoref{sec:lbacbf_filter}: $\eta=0$
certifies the best-fitting model, while $\eta=0.01$ and $\eta=0.05$ retain
sets of tens of candidates out of M, whose size varies from step to step,
widening whenever the recent data separates the bank less sharply. Unless stated otherwise,
\lbacbf\ refers to $\eta=0$.

\begin{figure}
    \centering
    \includegraphics[width=\linewidth]{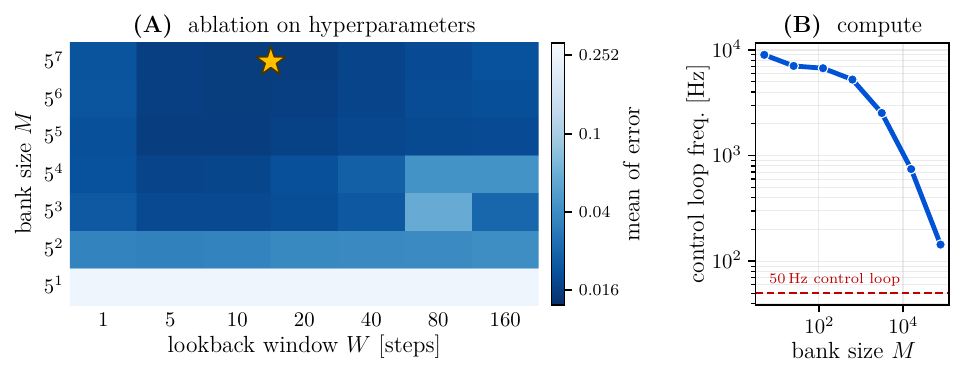}
    \caption{Ablation over bank size $M$ and window $W$. {\bf (A)}~One-step prediction error (star indicates minimum).
    {\bf (B)}~Control rate vs. $M$.}
    \label{fig:quad_sim_ablation}
    \vspace{-5mm}
\end{figure}

\subsubsection{Baselines and metrics}
We evaluate three adaptive CBF filters and one robust CBF filter, alongside
fixed-model and oracle references.
The disturbance-observer CBF
(\textcolor{cDOB}{\mbox{DOB-CBF}})~\cite{wang2023dobcbf}
low-pass filters the nominal model's planar force residual at bandwidth
$33\,\mathrm{s}^{-1}$.
The $\mathcal{L}_1$ adaptive CBF
(\textcolor{cLone}{\mbox{$\mathcal{L}_1$-AC}})~\cite{nguyen2022l1cbf}
estimates the same force using a predictor-based $\mathcal{L}_1$ adaptation
law and a $33\,\mathrm{rad/s}$ low-pass filter.
The high-order robust adaptive CBF
(\textcolor{cRaCBF}{\mbox{HO-RaCBF}})~\cite{cohen2022hocabf}
updates the force estimate by gradient descent within a priori bounds
$|w_x|,|w_y|\leq0.315\,\mathrm{N}$, with a robustness margin that contracts
with excitation.
All three adaptive baselines drive the same two-channel filter as
\lbacbf, retaining both the second-order thrust condition and the
fourth-order tilt condition. The robust CBF
(\textcolor{cRCBF}{\mbox{R-CBF}})~\cite{jankovic2018rcbf}
applies worst-case robustification to the second-order thrust condition
using bounds on wind and drag.
The fixed-model \textcolor{cNom}{\mbox{Nominal~CBF}} assumes
$\boldsymbol{\theta}=\mathbf{0}$, whereas the
\textcolor{cOracle}{\mbox{Oracle}} receives the true
$\boldsymbol{\theta}_k$ at every step. The three adaptive baselines use additive planar-force uncertainty
models, which support their observer constructions and the
linear-in-parameter formulation of
\textcolor{cRaCBF}{\mbox{HO-RaCBF}}.
These implementations do not explicitly identify the payload-dependent
input gain. In contrast, \lbacbf's model bank includes the payload
parameter, allowing its predictions to account for changes in control
effectiveness. The experiment thus evaluates the complete adaptation
designs under simultaneous wind variation and payload mismatch.

A run is \emph{safe} if $h\geq0$ for all ten obstacles, and \emph{reached} if its terminal position lies within
$0.2\,\mathrm{m}$ of the goal.
A \emph{success} requires both safety and reaching the goal.

The cumulative intervention is
$\Delta U=\sum_k\|\mathbf{u}_k-\mathbf{u}_k^{\mathrm{nom}}\|^2$.
Because this quantity accumulates along each method's trajectory,
we compare it on the common subset of 13 initial conditions for which
every method with nonzero success rate completed the task safely.

\begin{table}[t]
\centering
\caption{Quadrotor simulation, 26 random initial conditions per method.
    A \emph{success} is both safe and reached. $\Delta U$ is averaged over the 13
    starts on which every method with a nonzero success rate
    succeeded; n/a otherwise. $\eta$ is the plausible-set tolerance.}
\label{tab:quad_sim_methods}
\footnotesize
\setlength{\tabcolsep}{4pt}
\begin{tabular}{l
    S[table-format=3.0]
    S[table-format=3.0]
    S[table-format=3.0]
    S[table-format=3.1]}
\toprule
Method & {Safe [\%]} & {Reached [\%]} & {Success [\%]} & {$\Delta U$} \\
\midrule
\textcolor{cOracle}{Oracle CBF (truth)}     & 100 & 100 & 100 &  47.2 \\
\lbacbf\ \textcolor{lbacbf}{($\eta=0$)}                          & 100 & 100 & 100 &  42.6 \\
\lbacbf\ \textcolor{lbacbf}{($\eta=0.01$)}                       & 100 & 100 & 100 &  32.7 \\
\lbacbf\ \textcolor{lbacbf}{($\eta=0.05$)}                       & 100 & 100 & 100 &  31.6 \\
\textcolor{cRaCBF}{\mbox{HO-RaCBF}}                &  88 & 100 &  88 & 395.2 \\
\textcolor{cLone}{$\mathcal{L}_1$-AC + CBF} &  69 &  92 &  62 &  61.1 \\
\textcolor{cDOB}{\mbox{DOB-CBF}}                   &  62 & 100 &  62 &  73.4 \\
\textcolor{cRCBF}{\mbox{R-CBF}}                    & 100 &   0 &   0 & {\textcolor{nomGray}{n/a}} \\
\textcolor{cNom}{\mbox{Nominal~CBF}}               &  46 &   0 &   0 & {\textcolor{nomGray}{n/a}} \\
\bottomrule
\end{tabular}
\vspace{-10pt}
\end{table}

\begin{figure*}
    \centering
    \includegraphics[width=\linewidth]{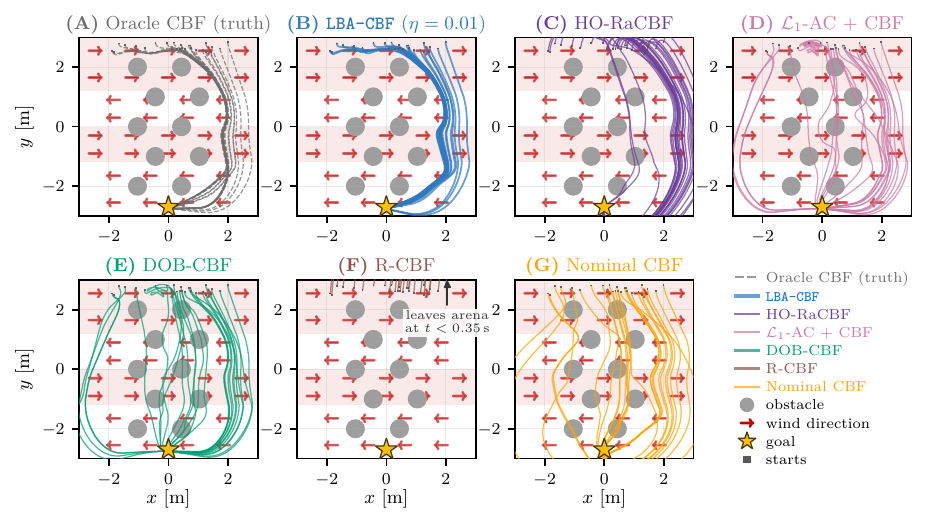} \vspace{-7mm}
    \caption{Quadrotor simulation: all 26 trajectories per method, ordered
    as in \autoref{tab:quad_sim_methods}. The vehicle descends from the grey squares to the goal (star) through five gates of paired obstacles
    (grey discs, whose edge is $h=0$); the wind blows along the red arrows,
    $+x$ in the shaded bands and $-x$ in the white ones.
    \textcolor{cOracle}{{\bf (A)}~\mbox{Oracle}}, \textcolor{lbacbf}{{\bf (B)}}~\lbacbf, \textcolor{cRaCBF}{{\bf (C)}~\mbox{HO-RaCBF}},
    \textcolor{cLone}{{\bf (D)}~\mbox{$\mathcal{L}_1$-AC}}, \textcolor{cDOB}{{\bf (E)}~\mbox{DOB-CBF}},
    \textcolor{cRCBF}{{\bf (F)}~\mbox{R-CBF}}, \textcolor{cNom}{{\bf (G)}~\mbox{Nominal~CBF}}.}
    \label{fig:quad_sim_traj}
    \vspace{-5mm}
\end{figure*}

\subsubsection{Results}
\autoref{tab:quad_sim_methods} and \autoref{fig:quad_sim_traj} summarize all the runs. \lbacbf\ succeeds from every initial condition, matching the
\textcolor{cOracle}{\mbox{Oracle}}, while \textcolor{cLone}{\mbox{$\mathcal{L}_1$-AC}} and \textcolor{cDOB}{\mbox{DOB-CBF}} each succeed from $62\%$:
collisions in $8$ and $10$ runs, plus two \textcolor{cLone}{\mbox{$\mathcal{L}_1$-AC}} runs that stall,
safely, in the start band. \textcolor{cRaCBF}{\mbox{HO-RaCBF}} succeeds from $88\%$ but bows wide of
the obstacles (\autoref{fig:quad_sim_traj}(C)), at $9\times$ the intervention of
\lbacbf. All observed adaptive-baseline collisions were preceded by
CBF infeasibility near an obstacle and application of the
nominal fallback command. Neither \textcolor{cRCBF}{\mbox{R-CBF}} nor the \textcolor{cNom}{\mbox{Nominal~CBF}} ever reaches the goal, and
\textcolor{cRCBF}{\mbox{R-CBF}} is safe on every run only because it climbs out of the arena
within $0.35\,\mathrm{s}$ and never returns (\autoref{fig:quad_sim_traj}(F)). On the
paired set, the intervention of \lbacbf\ is comparable to the
\textcolor{cOracle}{\mbox{Oracle}}'s ($42.6$ vs.\ $47.2$), while
\textcolor{cDOB}{\mbox{DOB-CBF}} needs $1.7\times$ as much and
\textcolor{cLone}{\mbox{$\mathcal{L}_1$-AC}} $1.4\times$. 

\begin{table}[t]
\centering
\caption{Effect of the plausible-set tolerance $\eta$: median closest
    approach to any obstacle over the all runs, each row flown from the same
    course and start.}
\label{tab:eta_margin}
\footnotesize
\setlength{\tabcolsep}{4pt}
\begin{tabular}{l S[table-format=2.1]}
\toprule
Method & {Closest distance to an obstacle [$\mathrm{cm}$]} \\
\midrule
\textcolor{cOracle}{Oracle CBF (truth)}      & 15.2 \\
\lbacbf\ \textcolor{lbacbf}{($\eta=0$)}      & 16.7 \\
\lbacbf\ \textcolor{lbacbf}{($\eta=0.01$)}   & 21.0 \\
\lbacbf\ \textcolor{lbacbf}{($\eta=0.05$)}   & 26.7 \\
\bottomrule
\end{tabular}
\vspace{-10pt}
\end{table}

\emph{Effect of $\eta$.} Because the tolerances succeed on nearly every start,
success rate cannot separate them, but the safety margin can. We fly the oracle
and \lbacbf\ on the original course and on randomly drawn courses that keep the
five gates but shift each corridor laterally by $0.15$--$0.45\,\mathrm{m}$ and each gate
row by up to $0.10\,\mathrm{m}$, each with its own $26$ random starts
(\autoref{tab:eta_margin}). The margin grows with $\eta$: $\eta=0.01$ keeps a
larger clearance than $\eta=0$ on $81\%$ of paired runs, and $\eta=0.05$ than
$\eta=0.01$ on $94\%$, confirming increasing robustness.

\begin{figure*}
    \centering
    \includegraphics[width=0.93\linewidth]{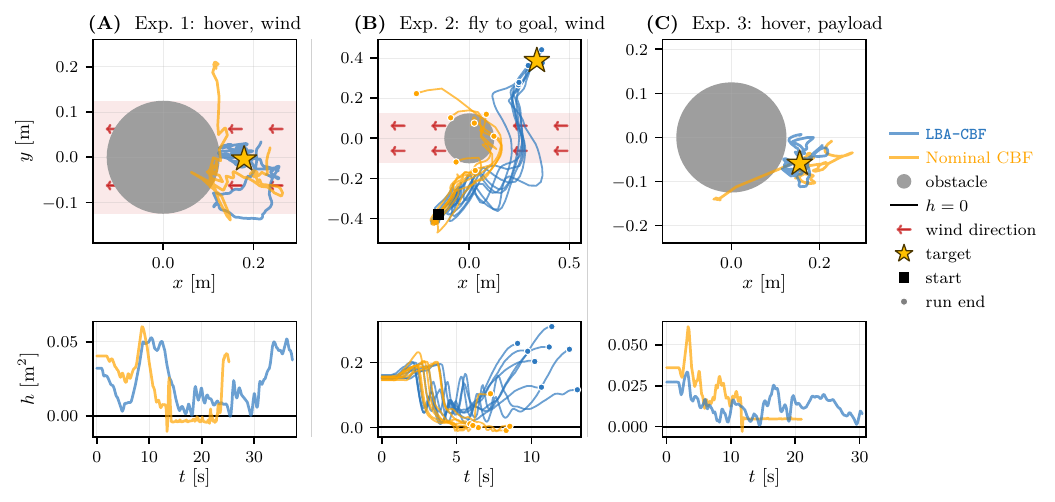}
    \caption{Crazyflie~2.1: \lbacbf\ vs.\ \textcolor{cNom}{\mbox{Nominal~CBF}}.
    {\bf (A)} hover and {\bf (B)} flight to a goal, both in wind; {\bf (C)} hover
    with a swinging payload. {\bf Top:} $xy$ trajectories, obstacle (grey disc),
    wind (red arrows), target (star). {\bf Bottom:} barrier value $h(t)$; $h<0$ is a
    violation.}
    \label{fig:quad_hardware_experiments}
    \vspace{-5mm}
\end{figure*}

\begin{figure}
    \centering
    \includegraphics[width=0.48\textwidth]{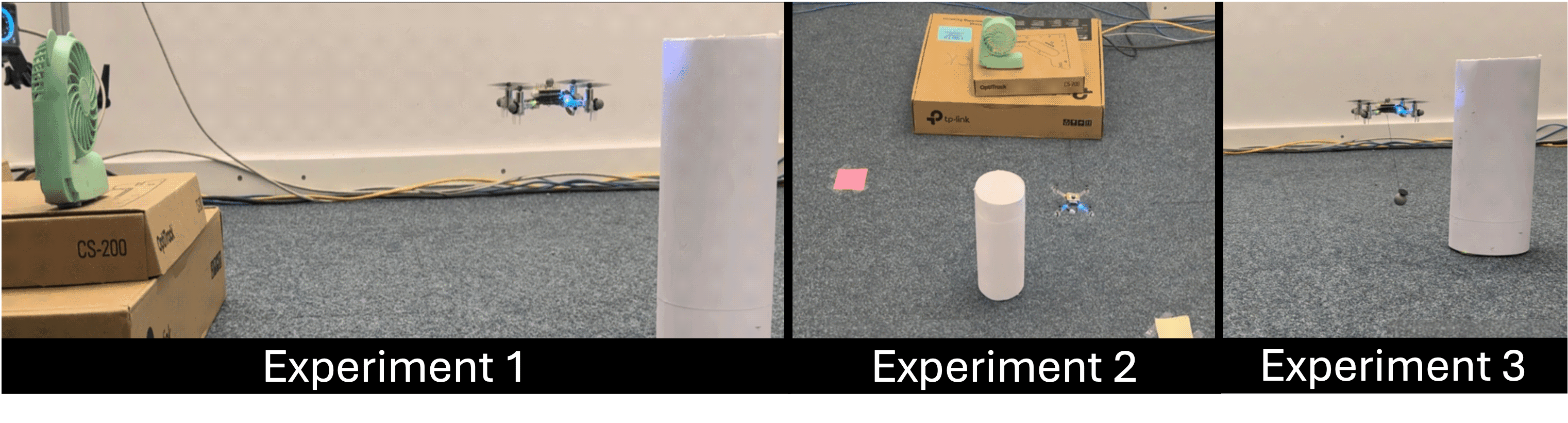}
    \caption{Crazyflie~2.1 hardware setups for Experiments~1--3.}
    \label{fig:hardware_experiments}
    \vspace{-5mm}
\end{figure}

\emph{Ablation.} \autoref{fig:quad_sim_ablation} varies
$W$ and $M$
at $\eta=0.01$, using the same initial conditions and
process noise of $0.01\,\mathrm{m/s}$ per step on $(v_x,v_y)$.
The smallest bank ($M=5$) has persistently high prediction
error. Expanding the bank to include $\Delta m$ ($M=25$)
reduces the error, while additional parameter axes provide
little improvement for $W\geq5$. The window length balances noise sensitivity and adaptation
speed. At $W=1$, the ranking depends on a single noisy transition; longer windows retain pre-change data for longer.

\emph{Compute.} On a CPU (Apple M1 Pro, JAX), the bank update
takes $0.95\,\mathrm{ms}$ at the deployed $M=11,341$ and $6.9\,\mathrm{ms}$ at
$M=78,125$---$1.1\,\mathrm{kHz}$ and $145\,\mathrm{Hz}$.

\subsection{Quadrotor: Hardware Experiments}
\label{subsec:quad_hw}

We deploy \lbacbf\ on a Crazyflie~2.1~\cite{giernacki2017crazyflie} tracked by a motion-capture system. An LQR running at $50\,\mathrm{Hz}$ produces the nominal horizontal
accelerations $\mathbf{u}_{\mathrm{nom}}$, which the safety filter modifies so
that the vehicle avoids a circular obstacle (\autoref{fig:hardware_experiments}).
The filter uses \eqref{eq:quadrotor_dynamics} in the horizontal plane around an
upright cylinder: no gravity term, no drag, and the commanded horizontal
accelerations in place of the rotor thrusts. The unknown parameters are the
planar external forces $\boldsymbol{\theta}=(w_x,w_y)$, used for both wind and
payload.

The bank is a uniform $500\times500$ grid over $(w_x,w_y)$, $M=250,000$,
spanning $[-0.2,0.2]^2\,\mathrm{N}$ in Experiments~1 and~2 and $[-0.1,0.1]^2\,\mathrm{N}$ in
Experiment~3, where the payload produces smaller forces; a grid covers the
unknown wind and payload intensities uniformly, with a look-back window $W=3$ throughout. In all experiments, we set $\eta=0$, selecting the best-fit endpoint of \lbacbf, which keeps a smaller set of barrier constraints, allowing us to isolate the look-back adaptation mechanism while minimizing the online optimization cost.

\textit{Experiment~1 (hover under wind).} The vehicle hovers beside the
obstacle while the fan pushes it toward the unsafe set;
see \autoref{fig:quad_hardware_experiments}(A). With the \textcolor{cNom}{\mbox{Nominal~CBF}}, the
vehicle hovered for only approximately $8\,\mathrm{s}$ before crashing into the
obstacle. With \lbacbf, it hovered for more than
twice as long with the fan running, and $h$ remained at or above zero.

\textit{Experiment~2 (flight to a goal under wind).} The vehicle flies past the
obstacle to a goal while the fan pushes it toward the obstacle;
see \autoref{fig:quad_hardware_experiments}(B). Across eight repetitions,
\lbacbf\ kept $h\geq0$ throughout, whereas the \textcolor{cNom}{\mbox{Nominal~CBF}} crossed $h=0$ in six
of the eight runs shown.

\textit{Experiment~3 (hover with a swinging payload).} A mass hangs from the
vehicle on a rope; see \autoref{fig:quad_hardware_experiments}(C). Because the
payload swings, the disturbance never settles, so the filter must follow a
continuously changing regime rather than converge to one. After confirming that
\lbacbf\ maintained safety in this condition, we cut the rope in flight,
imposing a second, abrupt change of the dynamics. The same test was not possible
with the \textcolor{cNom}{\mbox{Nominal~CBF}}, which had already violated the constraint with the payload
attached.

\subsection{Safe Extreme Driving: Hardware Experiments}
\label{subsec:f1tenth}
We deploy \lbacbf\ on an F1TENTH autonomous racing platform as the model used inside a CBF-constrained MPC, and compare it with the same controller using a fixed nominal model identified on rubber tires. As in~\cite{liao2026llampc_embedded}, the vehicle follows a single-track dynamic bicycle model with Fiala tire forces~\cite{zhang2018fiala}, parameterized by the front and rear cornering stiffnesses and friction coefficients, $\theta=[C_f,C_r,\mu_f,\mu_r]^\top$. The bank contains $10\,000$ models over $C_f\in[175,325]\,\mathrm{N/rad}$, $C_r\in[150,300]\,\mathrm{N/rad}$, and $\mu_f,\mu_r\in[0.1,1.1]$, and is evaluated in JAX with $\eta=0$. The same bank is used for both tire conditions, while the nominal model remains fixed. The MPC is implemented in acados~\cite{acados2021} with horizon $P=20$ and $dt=0.04\,\mathrm{s}$, and runs at $25\,\mathrm{Hz}$. We report model quality as the $P$-step prediction error: the deviation between the state obtained by propagating a model over the MPC horizon under the applied inputs and the measured state.

\textit{Experiment 1 (plastic tires, adaptive drifting)}: With \lbacbf, the car is able to stabilize over 6 laps even with consistent tire slippage with minimal conservatism (\autoref{fig:oval_traj}). The nominal model is overconservative in its optimization, causing it to deviate significantly from the track and crash.

\begin{figure}[th]
    \centering
    \vspace{-2mm}
    \includegraphics[width=.49\linewidth]{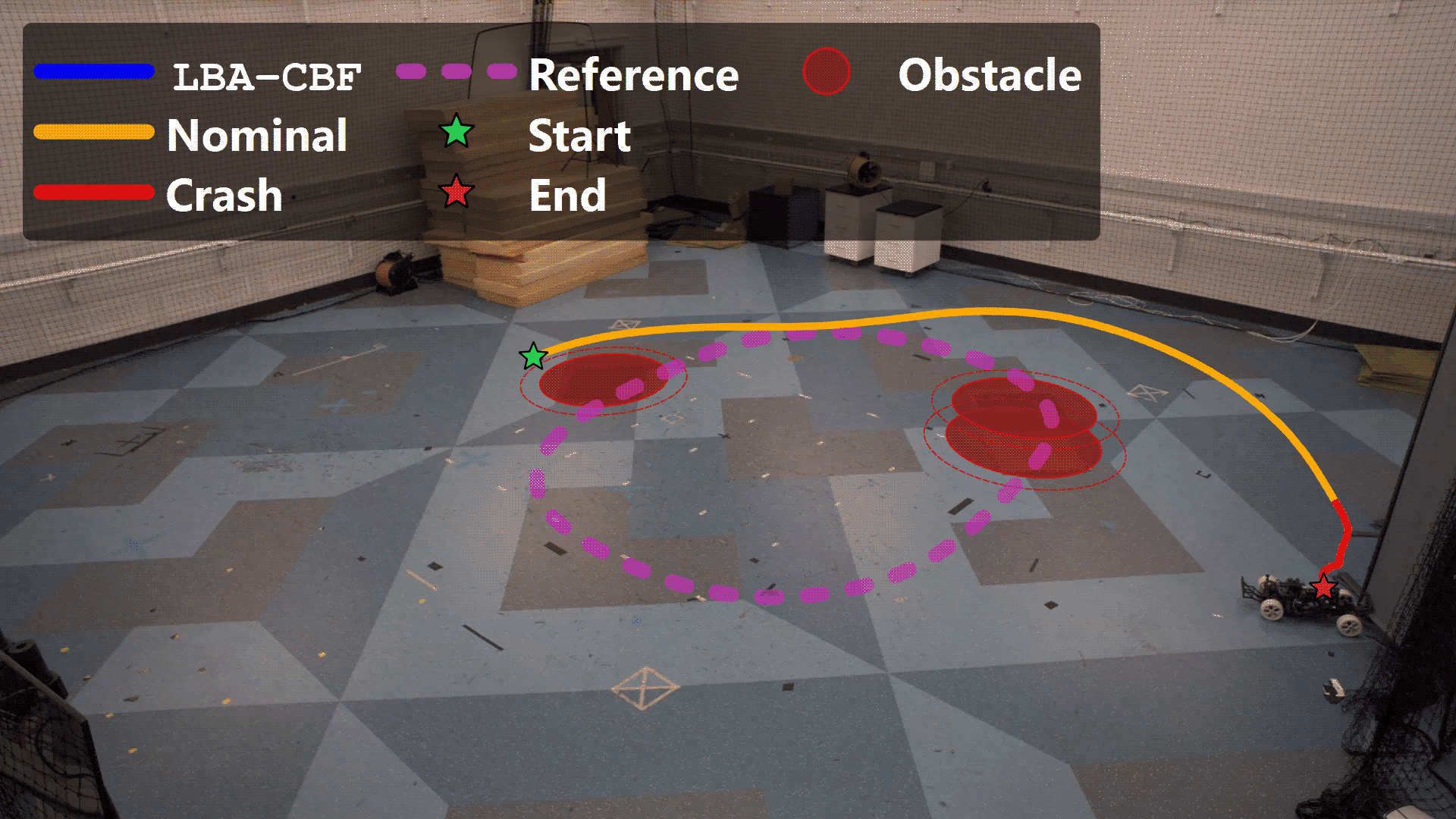}
    \includegraphics[width=.49\linewidth]{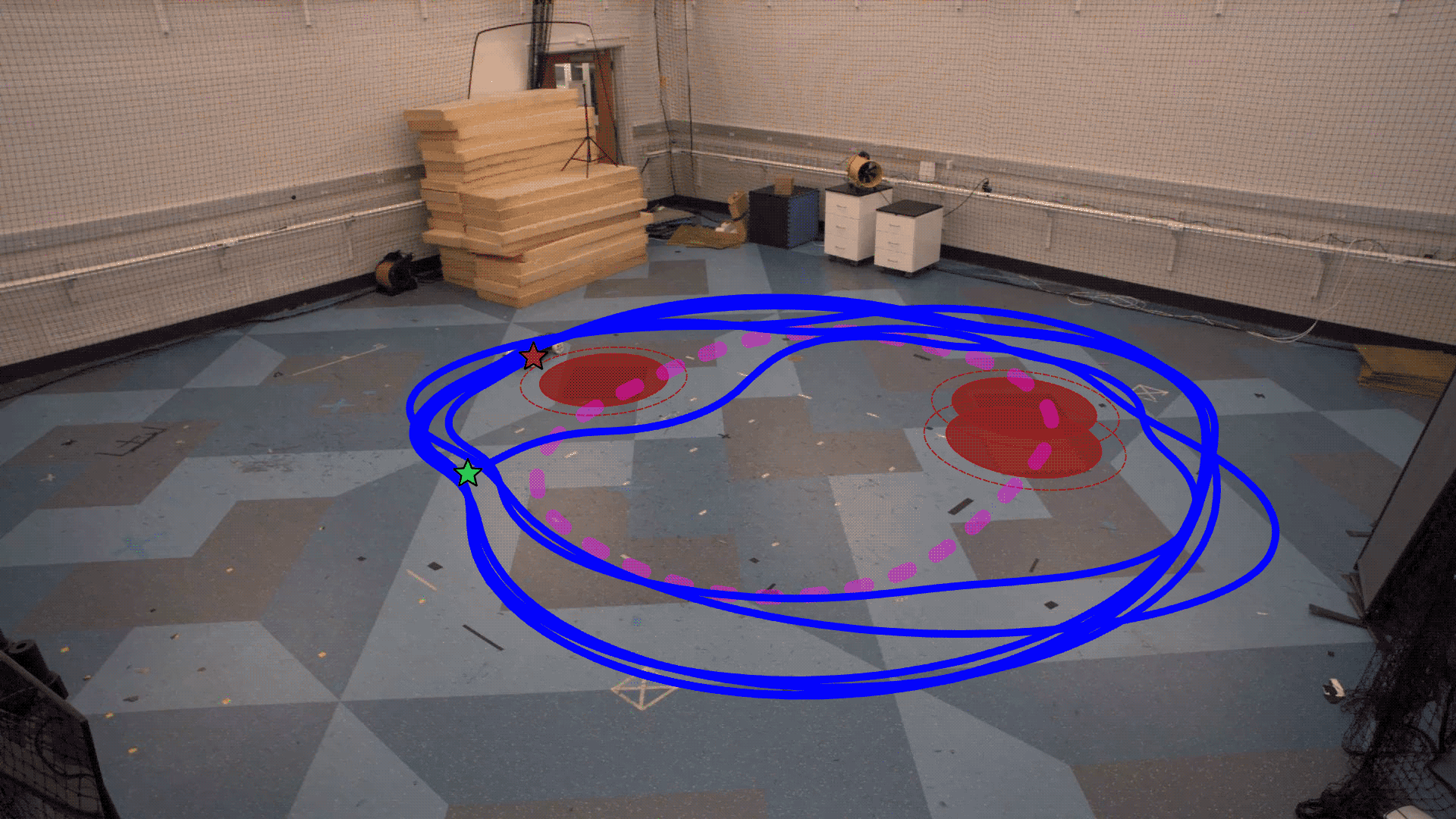}\\[1pt]
    \includegraphics[width=.49\linewidth]{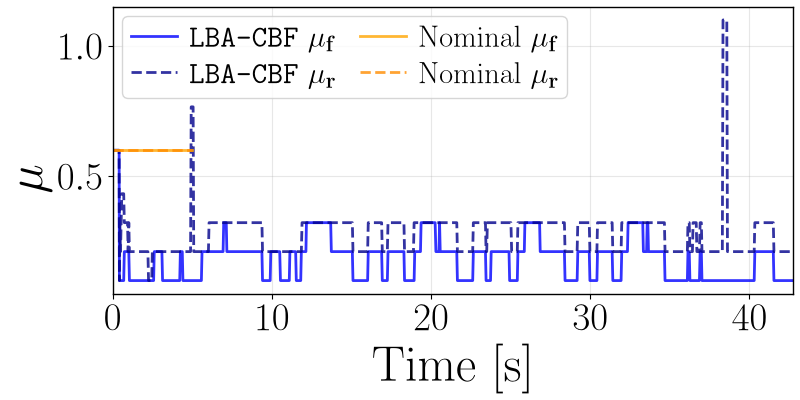}
    \includegraphics[width=.49\linewidth]{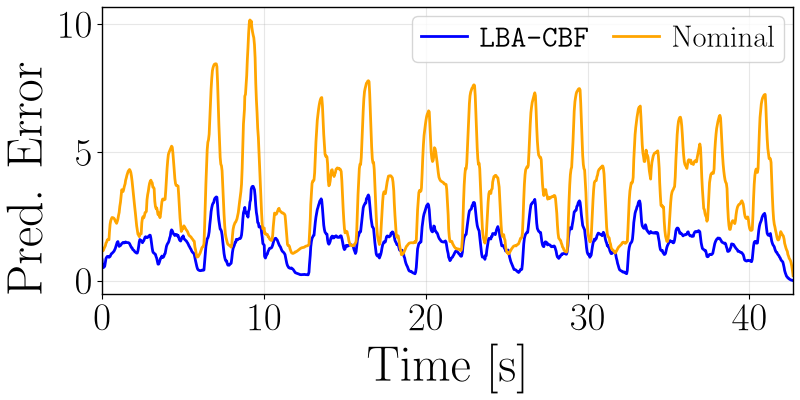}
    \caption{Experiment~1: nominal-model tracking (top left), \lbacbf\ tracking
    (top right), and \lbacbf's friction estimates and $P$-step prediction error
    (bottom).}
    \label{fig:oval_traj}
    \vspace{-7pt}
\end{figure}

In both experiments the friction estimates for \lbacbf{} converge, and the \lbacbf\ model achieves substantially lower prediction error than the nominal model \autoref{fig:oval_traj},~\autoref{fig:fig8_traj}.

\textit{Experiment 2 (rubber tires, inconsistent acceleration)}: With \lbacbf, the car is able to stabilize over 5 laps with minimal conservatism despite some tire slippage and intermittent acceleration and deceleration (\autoref{fig:fig8_traj}). The nominal model is again overconservative, and significant departure from the reference trajectory causes a crash.

\begin{figure}[th]
    \centering
    \vspace{-2mm}
    \includegraphics[width=.49\linewidth]{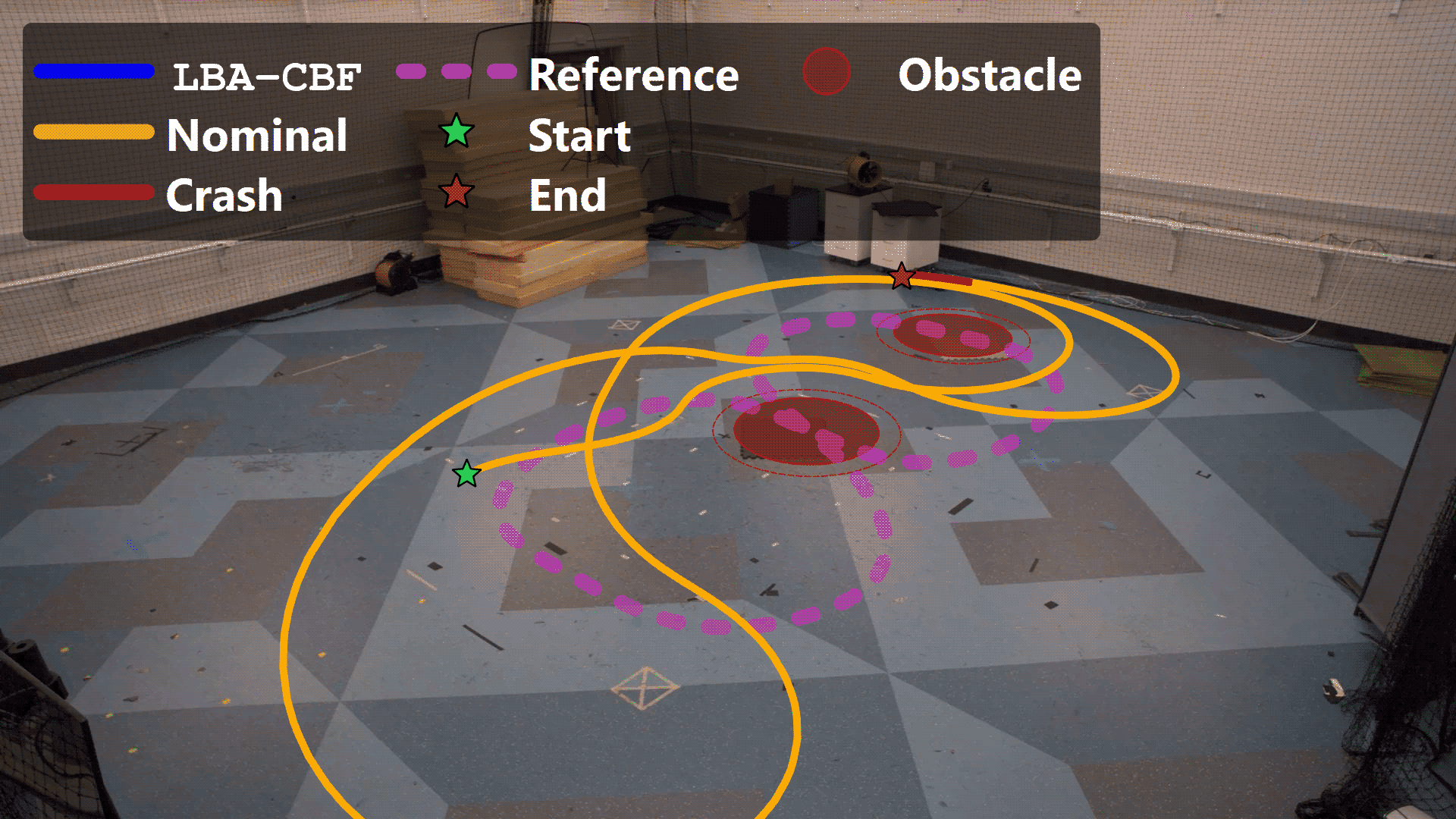}
    \includegraphics[width=.49\linewidth]{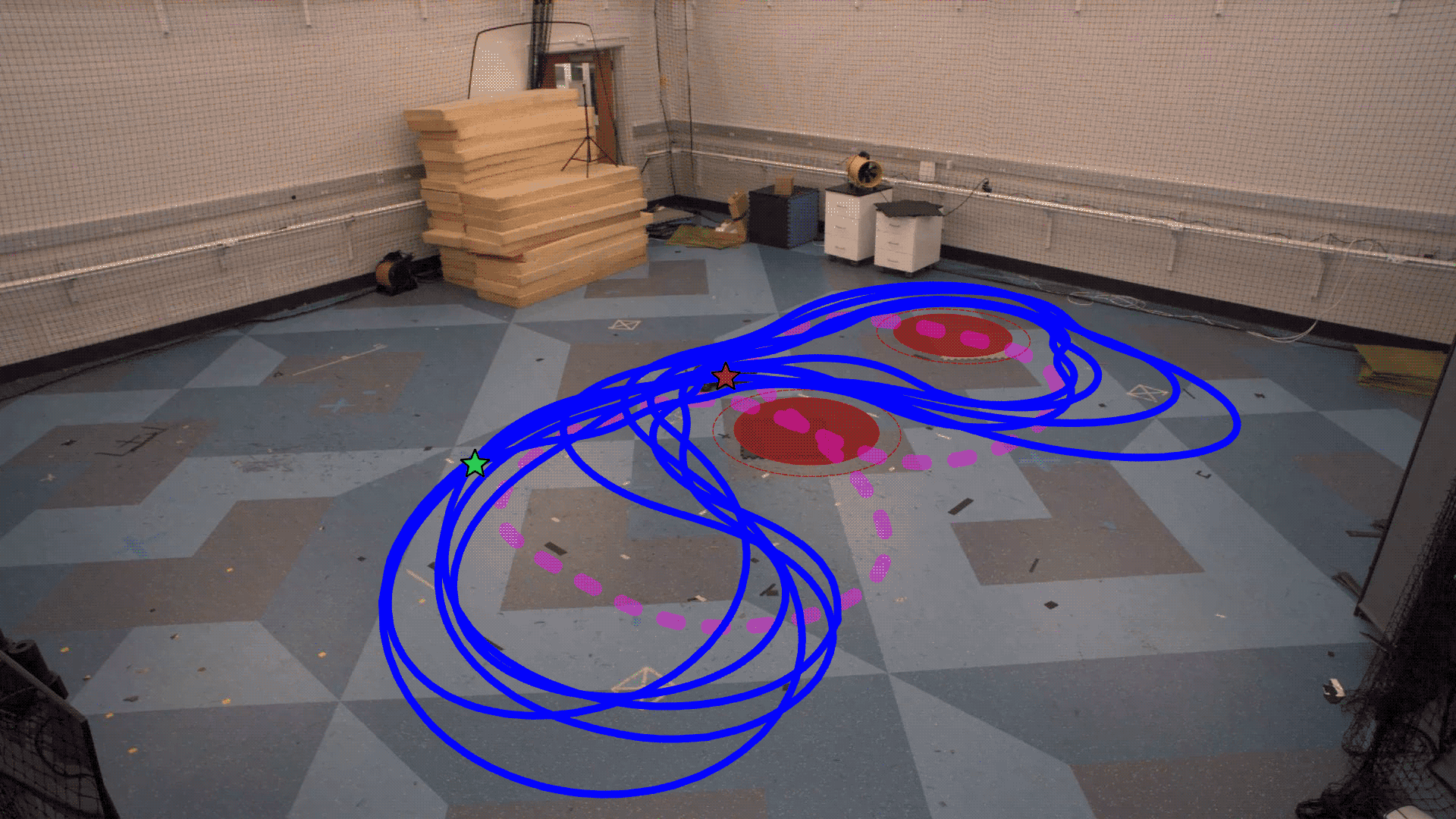}\\[1pt]
    \includegraphics[width=.49\linewidth]{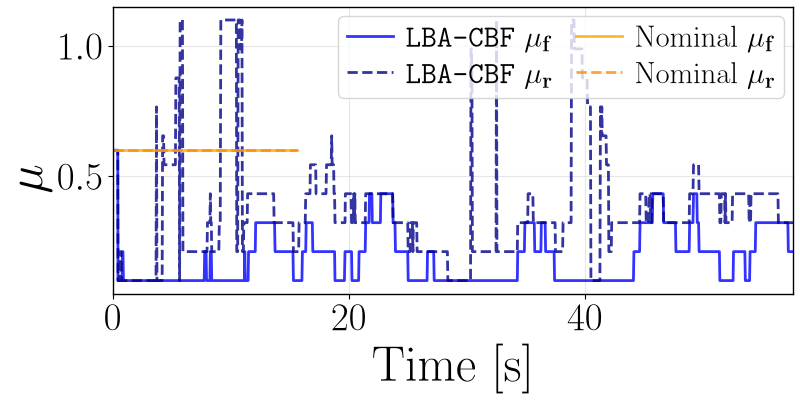}
    \includegraphics[width=.49\linewidth]{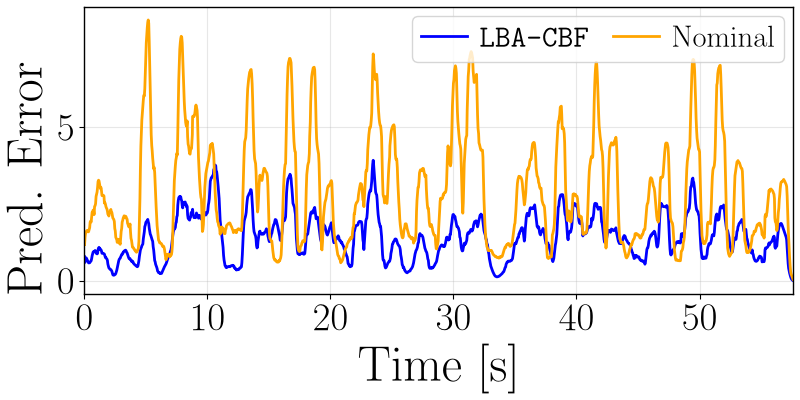}
    \caption{Experiment~2: nominal-model tracking (top left), \lbacbf\ tracking
    (top right), and \lbacbf's friction estimates and $P$-step prediction error
    (bottom).}
    \label{fig:fig8_traj}
\end{figure}

\vspace{-6mm}
\section{Conclusions}
We introduced \lbacbf, a safety filter that uses recent prediction error to decide which dynamics hypotheses the barrier must be certified against: a finite bank is ranked over a look-back window, and a single tolerance spans best-fit adaptation and full-bank robust filtering. In simulation and on Crazyflie~2.1 and F1TENTH hardware, it adapts at control rate to abrupt wind reversals, an unmodeled payload, and changes of tire, where fixed-model, robust, and continuously adapting filters become unsafe or overly conservative. The guarantee remains conditional on bank adequacy and on filter feasibility; constructive conditions for these, and online adaptation of the tolerance, are left to future work.

\printbibliography
\end{document}